\documentclass[runningheads]{llncs}
\usepackage[T1]{fontenc}
\usepackage{graphicx}
\usepackage{amsmath}
\usepackage{amssymb}
\usepackage{bm}
\usepackage{booktabs}
\usepackage{hyperref}
\usepackage{cite}  
\usepackage{float}
\usepackage{placeins}
\usepackage{color}
\usepackage{subcaption}

\newcommand{\R}{\mathbb{R}}
\newcommand{\C}{\mathbb{C}}
\newcommand{\N}{\mathbb{N}}

\begin{document}
\title{Approximating Matrix Functions with Deep Neural Networks 
       and Transformers}
\titlerunning{Matrix Functions with DNNs and Transformers}
\author{Rahul Padmanabhan\inst{1}\orcidID{0009-0005-8660-6234} \and \\
        Simone Brugiapaglia\inst{1}\orcidID{0000-0003-1927-8232}}
\authorrunning{R. Padmanabhan and S. Brugiapaglia}
\institute{Department of Mathematics and Statistics, 
           Concordia University, \\
           1455 Blvd. De Maisonneuve Ouest, Montreal, Quebec H3G 1M8, Canada\\
           \email{\{rahul.padmanabhan,simone.brugiapaglia\}@concordia.ca}}
\maketitle
\begin{abstract}
Transformers have revolutionized natural language processing, but their 
use for numerical computation has received less attention. 
We study the approximation of matrix functions, which map scalar 
functions to matrices, using neural networks including transformers. 
We focus on functions mapping square matrices to square matrices of the same dimension.
These types of matrix functions appear throughout scientific computing, e.g., 
the matrix exponential in continuous-time Markov chains and the
matrix sign function in stability analysis of dynamical systems.
In this paper, we make two contributions. 
First, we prove bounds on the width and depth of ReLU networks needed 
to approximate the matrix exponential to an arbitrary precision. 
Second, we show experimentally that a transformer encoder-decoder with 
suitable numerical encodings can approximate certain matrix functions
at a relative error of 5\% with high probability.
Our study reveals that the encoding scheme strongly affects performance, with different 
schemes working better for different functions.

\end{abstract}

\section{Introduction}
\label{sec:introduction}

Matrix functions appear throughout scientific computing.
The matrix exponential gives transition probabilities 
in continuous-time Markov chains~\cite{ross2014introduction}, the matrix sign function 
solves algebraic Riccati equations in control theory~\cite{roberts1980linear}, 
and the matrix sine and cosine arise in second-order differential 
equations~\cite{al2015new}.
Given $A \in \C^{n \times n}$, a matrix function $f(A)$ produces 
another $n \times n$ matrix.
These functions do not correspond to element-wise operations, but
they respect the spectral structure of $A$.
Standard algorithms for computing these functions can be expensive,
(especially for large matrices),
which motivates the search for fast surrogate models.

Neural networks are natural candidates for such surrogates.
Deep learning has proven effective in vision~\cite{voulodimos2018deep} 
and language~\cite{wu2020deep,otter2020survey}, and recent work has begun 
applying neural networks to mathematical tasks.
Transformers~\cite{vaswani2017attention}, which underlie models like 
Llama~\cite{touvron2023llama}, ChatGPT~\cite{achiam2023gpt}, 
Claude~\cite{anthropic2024claude3}, Grok~\cite{xai2024grok1repo}, and 
DeepSeek~\cite{deepseek2024}, have shown promise for symbolic 
integration~\cite{lample2019deep}, linear algebra~\cite{charton2021linear}, 
and finding Lyapunov functions~\cite{alfarano2024global}.
But matrix functions present a different challenge: the output depends 
on the eigenstructure of the input, and the mapping from entries to 
entries is highly nonlinear.
Can neural networks learn these mappings accurately?

We study this question for five matrix functions (exponential, logarithm, 
sign, sine, and cosine) using both feedforward networks and transformers.
Our experiments show that standard deep networks struggle beyond $1 \times 1$ 
matrices, but transformer encoder-decoders with appropriate numerical 
encodings can approximate certain functions reasonably well.

\vspace{0.15cm}
\noindent\textbf{Summary of Contributions.} \label{subsec:contributions}
The contribution of this paper is twofold. From a theoretical perspective
we prove that a ReLU feedforward network can approximate $e^A$ to 
precision $\epsilon$, with width exponential in $nM$ and 
depth linear in $nM$ (Theorem~\ref{thm:dnn_matrix_exp}).
Moreover, numerically, we show that a transformer encoder-decoder with numerical 
encodings can approximate the matrix sign function with 
88.41\% accuracy for $3\times 3$ matrices 
at 1\% tolerance.

\vspace{0.15cm}
\noindent\textbf{Organization.} We start with Section~\ref{sec:preliminaries},
where we provide some preliminaries on matrix functions, transformers, and
the encoding scheme used in the experiments.
In Section~\ref{sec:theory}, we prove the theoretical result on the width and depth
of a ReLU feedforward network to approximate the matrix exponential.
In Section~\ref{sec:experiments}, we present the numerical results on
the transformer encoder-decoder with numerical encodings to approximate
the matrix sign function.
Finally, we conclude in Section \ref{sec:conclusions} with a discussion of
the results and future directions.

\section{Preliminaries}
\label{sec:preliminaries}
This section provides some preliminaries on matrix functions, transformers, and the encoding scheme used in the experiments.

\vspace{0.15cm}
\noindent\textbf{Matrix Functions.} Matrix functions extend scalar functions to matrices while preserving 
spectral properties.
Higham~\cite{higham2008functions} is a comprehensive reference for matrix functions.
We use the Jordan canonical form definition as, while theoretical,
it is a simple and intuitive definition that can be used to understand the behavior of matrix functions.

\begin{definition}[Matrix Function via Jordan Canonical Form]
\label{def:jordan}
Let $f$ be defined on the spectrum of $A \in \C^{n \times n}$ and let 
$A = ZJZ^{-1}$ where $J = \operatorname{diag}(J_1, \ldots, J_p)$ is 
the Jordan form. Then $f(A) := Z\operatorname{diag}(f(J_1), \ldots, f(J_p))Z^{-1}$,
where, for a Jordan block $J_k$ with eigenvalue $\lambda_k$ of size $m_k$, we define
\begin{equation*}
    f(J_k) := \begin{bmatrix} 
        f(\lambda_k) & f'(\lambda_k) & \cdots 
            & \frac{f^{(m_k-1)}(\lambda_k)}{(m_k-1)!} \\
        & f(\lambda_k) & \ddots & \vdots \\
        & & \ddots & f'(\lambda_k) \\
        & & & f(\lambda_k)
    \end{bmatrix}.
\end{equation*}
\end{definition}

One of the assumptions made in Definition~\ref{def:jordan} is that the
function $f$ is $m_k - 1$ times differentiable at the eigenvalue $\lambda_k$.
We will now briefly describe the five matrix functions used in the experiments:

\begin{itemize}
\item \textit{Matrix Exponential.} 
For $A \in \mathbb{C}^{n \times n}$, we have $e^A = \sum_{k=0}^{\infty} \frac{A^k}{k!}$.
This solves $\dot{x} = Ax$ with $x(t) = e^{At}x(0)$.

\item \textit{Matrix Logarithm.} 
For $A \in \mathbb{C}^{n \times n}$ with no eigenvalues on $\mathbb{R}^-$,
the principal logarithm satisfies $e^{\log(A)} = A$. When $\|A - I\| < 1$,
it can be computed via $\log(A) = \sum_{k=1}^{\infty} (-1)^{k+1} \frac{(A - I)^k}{k}$.
This is used in robotics in the context of interpolation of motion~\cite{alexa2002linear}.

\item \textit{Matrix Sign.} 
For $A \in \mathbb{C}^{n \times n}$ with no purely imaginary eigenvalues, 
write $A = Z\operatorname{diag}(J_1, J_2)Z^{-1}$ where eigenvalues 
of $J_1$ have negative real parts and those of $J_2$ positive real parts.
Then $\operatorname{sign}(A) = Z\operatorname{diag}(-I, I)Z^{-1}$.
This is a key function used in solving algebraic Riccati equations~\cite{roberts1980linear} and Lyapunov equations~\cite{higham2008functions}.

\item \textit{Matrix Sine and Cosine.} 
For $A \in \mathbb{C}^{n \times n}$, we have $\sin(A) = \sum_{k=0}^{\infty} (-1)^k \frac{A^{2k+1}}{(2k+1)!}$
and $\cos(A) = \sum_{k=0}^{\infty} (-1)^k \frac{A^{2k}}{(2k)!}$.
These are used in second-order differential equations~\cite{al2015new}.
\end{itemize}

\vspace{0.15cm}
\noindent\textbf{Transformer Architecture.}
\label{subsec:transformer}
Transformers~\cite{vaswani2017attention} process sequences via 
self-attention.
The key idea is that each position in the sequence can attend to all other 
positions, with learned weights that capture which inputs are relevant for 
producing each output.
For input $X \in \R^{n \times d}$, 
    $\mathrm{Attention}(Q, K, V) = \mathrm{softmax}\left(\frac{QK^T}{\sqrt{d_k}}\right)V$,
where $Q = XW_Q$, $K = XW_K$, $V = XW_V$ are linear projections of the input.
Here $Q$, $K$, $V$ are the queries, keys, and values: the query at each position 
is compared against all keys to determine how much weight to give each value.
Multi-head attention runs $h$ parallel operations:
$\mathrm{MultiHead}(Q, K, V) = \mathrm{Concat}(\text{head}_1, \ldots, \text{head}_h)W_O$.
The matrices $W_Q, W_K, W_V, W_O$ are learned parameters.

\vspace{0.15cm}
\noindent\textbf{Encoding Schemes for Numerical Data.}
\label{subsec:encoding}
Following Charton~\cite{charton2021linear}, we represent floating-point 
numbers in the sign, mantissa, and exponent form
as $x \approx s \cdot m \cdot 10^e$ where $s \in \{-1,1\}$, 
$m \in \{100,\ldots,999\}$, $e \in \mathbb{Z}$.
Table~\ref{tab:encodings} shows four schemes.

\begin{table}[t]
\centering
\begin{tabular}{|l|c|c|c|c|}
\hline
\rule{0pt}{2.6ex}\textbf{Encoding} & \textbf{3.14} & \textbf{$-6.02 \times 10^{23}$} 
    & \textbf{Tokens/coef.} & \textbf{Vocab size} \\
\hline
\rule{0pt}{2.6ex}P10   & $[+, 3, 1, 4, \text{E-2}]$ & $[-, 6, 0, 2, \text{E21}]$ 
    & 5 & 210 \\
P1000 & $[+, 314, \text{E-2}]$     & $[-, 602, \text{E21}]$     
    & 3 & 1100 \\
B1999 & $[314, \text{E-2}]$        & $[-602, \text{E21}]$       
    & 2 & 2000 \\
FP15  & $[\text{FP314/-2}]$        & $[\text{FP-602/21}]$       
    & 1 & 30000 \\
\hline
\end{tabular}
\vspace{0.08cm}
\caption{Encoding schemes for floating-point 
         numbers~\cite{charton2021linear}.}
\label{tab:encodings}
\end{table}

\section{Theoretical Results: DNN Bounds for Matrix Exponential}
\label{sec:theory}

We bound the architecture of ReLU Deep Neural Networks (DNNs) for approximating the matrix 
exponential.
The proof uses the Taylor expansion and the following lemma 
from Adcock et al.~\cite{adcock2025near} (based on arguments from Schwab et al.~\cite{schwab2019deep}).

\begin{lemma}[Approximate Multiplication by ReLU DNNs {\cite[Lemma 7.1]{adcock2025near}}]
\label{lem:multiplication}
Let $0 < \delta < 1$, $l \in \N$, and $M = \prod_{i=1}^{l} M_i \geq 1$.
There exists a ReLU DNN $\chi^{(l)}_{\delta}$ with
\[
    \sup_{|x_i| \leq M_i} 
        \left| \prod\limits_{i=1}^{l} x_i - \chi^{(l)}_{\delta}(\mathbf{x}) \right| 
        \leq \delta,
\]
width $\leq c_1 \cdot l$, and 
depth $\leq c_2(1 + \log(l)[\log(l\delta^{-1}) + \log(M)])$, where $c_1, c_2 > 0$ are universal constants.
\end{lemma}

Lemma~\ref{lem:matrix_power} follows from repeated application of the definition of matrix multiplication: 
each entry $(A^k)_{ij}$ is a sum over all index paths from $i$ to $j$ of length $k$.
This can be proved by a standard induction argument, see Padmanabhan~\cite{padmanabhan2025deep}.
\begin{lemma}[Matrix Power Representation]
\label{lem:matrix_power}
For $A \in \C^{n \times n}$ and $k \in \N$, the entries of $A^k$ satisfy:
    $(A^k)_{ij} = \sum\limits_{\ell_1=1}^{n} \cdots \sum\limits_{\ell_{k-1}=1}^{n} 
                 \prod\limits_{q=1}^{k} a_{\ell_{q-1} \ell_q}$,
where $\ell_0 = i$ and $\ell_k = j$.
\end{lemma}

We now state our main theoretical result, which gives explicit bounds on the 
width and depth of a ReLU network that approximates the matrix exponential 
to arbitrary precision $\epsilon$.
The width grows like $K \cdot n^K \approx Mn \cdot n^{nM}$, 
exponential in $nM$.
The depth grows as $K\ln(K)$ times logarithmic factors, 
roughly linear in $Mn$. The complete proof can be found in Padmanabhan~\cite{padmanabhan2025deep}.

\begin{theorem}[DNN Architecture for Matrix Exponential]
\label{thm:dnn_matrix_exp}
Let $n \in \N$ and $M \geq 1$.
For any $\epsilon > 0$, there exists a ReLU network $f_\epsilon$ with
\begin{equation*}
    \sup_{A \in [-M,M]^{n \times n}} \|f_\epsilon(A) - e^A\|_F \leq \epsilon,
\end{equation*}
and bounds
\begin{align*}
    \operatorname{width}(f_\epsilon) 
        &\leq C_1 \cdot K \cdot n^{K}, \\
    \operatorname{depth}(f_\epsilon) 
        &\leq C_2 \left[1 + \ln(K)\left(\ln(K) 
              + \ln\left(\tfrac{2e}{\epsilon}\right) 
              + K(\ln(n) + \ln(M))\right)\right],
\end{align*}
where 
    $K = \left\lceil \max \left\{ enM, 
        \frac{nM + \ln(\frac{\sqrt{2}}{\sqrt{\pi}\epsilon})}{\ln(2)} 
        - 1 \right\} \right\rceil$,
$\|\cdot\|_F$ denotes the Frobenius norm, and $C_1, C_2 > 0$ are universal constants.
\end{theorem}

\begin{proof}[Sketch]
Here we provide a sketch of the proof.

\vspace{0.25cm}
\noindent\textbf{Step 1: Taylor truncation.}
Write $e^A = \sum\limits_{k=0}^{K} \frac{A^k}{k!} + R_K(A)$ where
\[
    \|R_K(A)\|_F \leq \frac{1}{\sqrt{2\pi}}\left(\frac{1}{2}\right)^{K+1}e^{nM}
\]
for $K \geq 2enM - 1$.

\vspace{0.25cm}
\noindent\textbf{Step 2: Remainder bound.}
Requiring $\|R_K(A)\|_F \leq \epsilon/2$ gives the condition
$K \geq \max\left\{enM, \frac{nM + \ln\left(\frac{1}{\sqrt{2\pi}\epsilon/2}\right)}{\ln(2)} - 1\right\}$.

\noindent\textbf{Step 3: Network construction.}
By Lemma~\ref{lem:matrix_power}, each $(A^k)_{ij}$ is a sum of $n^{k-1}$ 
products of $k$ terms.
Using Lemma~\ref{lem:multiplication}, we build networks $\mathcal{P}^{(k)}$ 
approximating $A^k$.
The network $\Phi(A) = \sum\limits_{j=0}^{K} \frac{1}{j!}\mathcal{P}^{(j)}(A)$ 
satisfies $\|\Phi(A) - \sum\limits_{j=0}^{K}\frac{A^j}{j!}\|_F \leq \delta e^n$ 
with $\delta = \epsilon/(2e^n)$.
\hfill $\square$
\end{proof}

\section{Numerical Experiments}
\label{sec:experiments}

We now run experiments with the objective of approximating the five matrix functions as
referred to in Section~\ref{sec:preliminaries}.
We test four architectures:
(1) shallow neural network with 3 hidden layers,
(2) deep neural network with 7 hidden layers,
(3) transformer encoder with Fourier features based on the work of Tancik et al.~\cite{tancik2020fourier}, and
(4) transformer encoder-decoder with numerical encodings based on the work of Charton~\cite{charton2021linear}.
The next section details the experimental setup for the baseline methods and the transformer encoder-decoder with numerical encodings.
Our experiments can be reproduced using the code available at \url{www.github.com/rahul3/LAWT}.

\subsection{Experimental Setup}
The baseline methods were chosen to compare the performance of more general neural network architectures
and an encoder only transformer architecture with the transformer encoder-decoder with numerical encodings.
Across all methods we use a tolerance-based accuracy metric to evaluate the performance of the models which is defined as follows:
\begin{definition}[Tolerance-Based Accuracy Metric] \label{def:tolerance_based_accuracy}
For tolerance $\tau$, let
\begin{equation}
    \mathrm{Accuracy}(\tau) = \frac{1}{N_{\mathrm{eval}}} \sum\limits_{i=1}^{N_{\mathrm{eval}}} \mathbf{1}\left[ \frac{\sum_{j,k} |(\hat{Y}_i)_{jk} - (Y_i)_{jk}|}{\sum_{j,k} |(Y_i)_{jk}| + \epsilon} < \tau \right]
\end{equation}
where $N_{\mathrm{eval}}$ is the number of evaluation samples, $Y_i$ is the true output, 
$\hat{Y}_i$ is the predicted output, $\epsilon$ is some small positive constant for numerical stability.
\end{definition}
In our experiments, we use $\epsilon = 10^{-7}$ and $\tau \in \{0.05, 0.02, 0.01, 0.005\}$.

\subsubsection{Baseline Methods.}
We used the following as the baseline methods:
\begin{itemize}
    \item The shallow neural network has hidden layers of size 128, 256, and 128 with ReLU activation function.
    \item The deep neural network has hidden layers of size 128, 256, 512, 1024, 512, 256, and 128 with ReLU activation function and dropout 0.2.
    \item The transformer encoder with Fourier features has 2, 4, 8, and 16 layers with $d^2$ attention heads where $d$ is the dimension of the matrix.
\end{itemize}
The training data for all cases above is sampled from a Gaussian distribution in $[-5,5]$ and
Adam~\cite{kingma2014adam} is used as the optimizer with a learning rate of $10^{-3}$.
The amount of training samples was varied from $2^5$ to $2^{18}$ in all the baseline methods.
The shallow and deep neural networks are trained for 100 epochs with a batch size of 128.
The transformer encoder with Fourier features is trained for 600 epochs with a batch size of 64
to give it enough time to learn the underlying structure of the data.
In the case of the transformer encoder with Fourier features, we follow the approach of
Tancik et al.~\cite{tancik2020fourier} where the inputs are mapped via 
$\gamma(x) = [\cos(2\pi Bx), 
                       \sin(2\pi Bx)]$,
where $B_{ij} \sim \mathcal{N}(0, \sigma^2)$ are independent.
We test 2, 4, 8, and 16 layers with up to $d^2$ heads where $d$ is the dimension of the matrix. 

\paragraph{Baseline Method Loss Functions.} \label{subsec:loss_functions}
The loss functions are selected based on the function being approximated.
For the exponential operation, due to the dynamic range of the output,
we use the $\ell_1$ relative error loss $\mathcal{L}_{\text{rel}} = \mathbb{E}[\|\hat{Y} - Y\|_1 / (\|Y\|_1 + \epsilon)]$ where $\epsilon=10^{-7}$,
to account for the large dynamic range of exponential outputs.
For logarithm, sign, sine, and cosine operations, we employ the Frobenius norm
loss $\mathcal{L}_{\text{Frob}} = \|Y - \hat{Y}\|_F$ for matrices ($d > 1$) and
mean squared error for scalars ($d = 1$).
The relative error formulation for exponentials ensures scale-invariant learning,
while the Frobenius norm naturally extends Euclidean distance to matrix spaces.

\paragraph{Transformer Encoder-Decoder.}
The main architecture has 8 encoder layers, 1 decoder layer, 
8 attention heads, embedding dimension 512.
We train for 100 epochs, using 300K samples/epoch, batch size 64, 
Adam with warmup (10,000 steps), peak learning rate of $10^{-4}$.
Matrix entries for the random matrices considered are sampled from a normal distribution, then clipped to the interval $[-5,5]$.

\subsection{Baseline Results}
The baseline results are shown in Figures~\ref{fig:snn_max_accuracy_by_dimension} and~\ref{fig:dnn_max_accuracy_by_dimension}
for the 3 layer shallow neural network and the 7 layer deep neural network respectively.
In the case of the Fourier-based transformer encoder, we start with the evaluation of dimension 3
and find that the results are unsatisfactory with approximately 0\% accuracy across all functions.
The highest tolerance-based accuracy is 0.256\% for the matrix sign function at dimension 3 at a tolerance of 0.05
when the Fourier-based transformer encoder has 4 layers.
This motivates the use of the encoder-decoder with numerical encodings, which is the main focus of the next section.

\begin{figure}[htbp]
    \centering
    \begin{subfigure}[b]{0.48\textwidth}
        \centering
        \includegraphics[width=\textwidth]{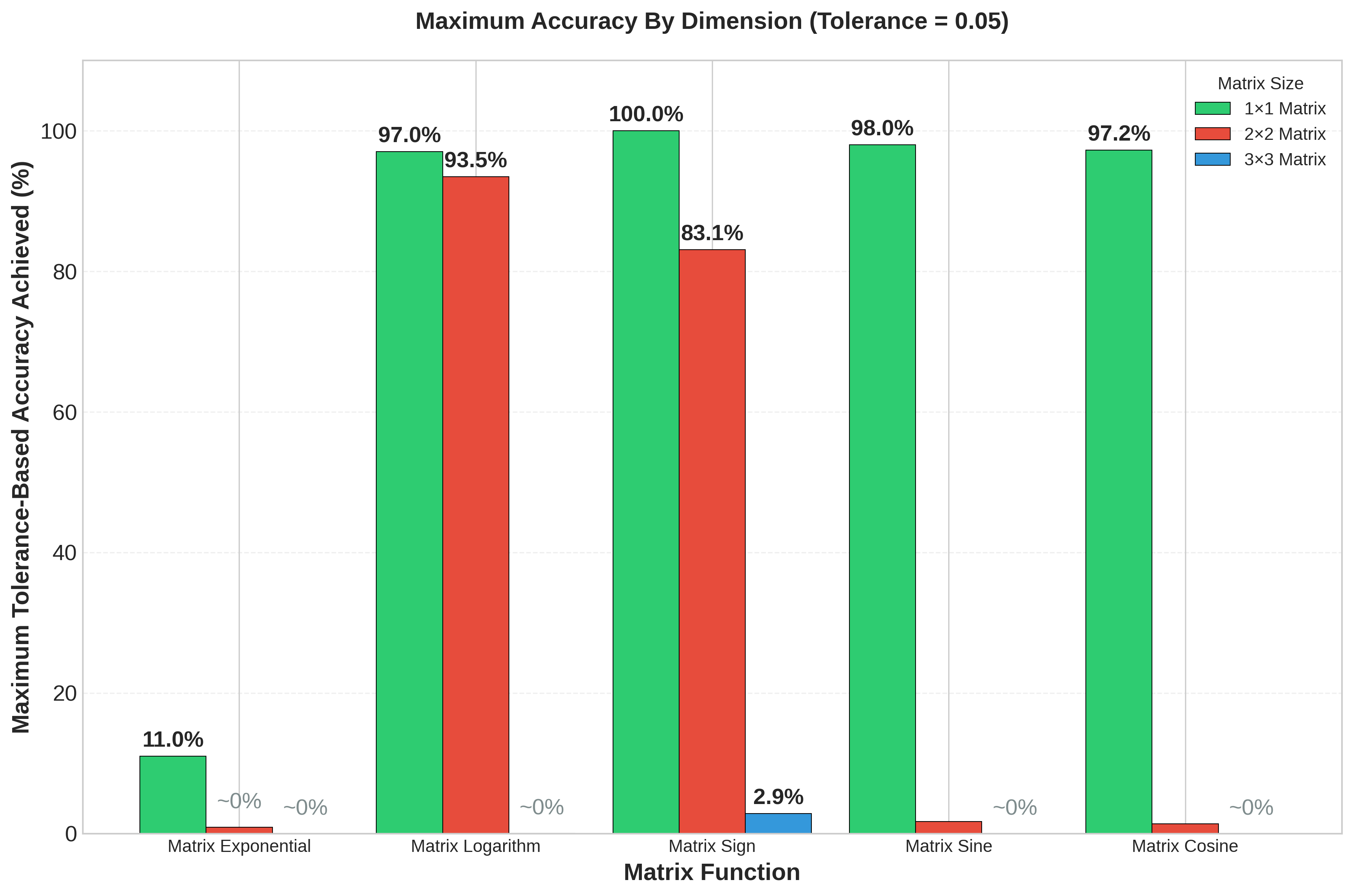}
        \caption{3-layer shallow network}
        \label{fig:snn_max_accuracy_by_dimension}
    \end{subfigure}
    \hfill
    \begin{subfigure}[b]{0.48\textwidth}
        \centering
        \includegraphics[width=\textwidth]{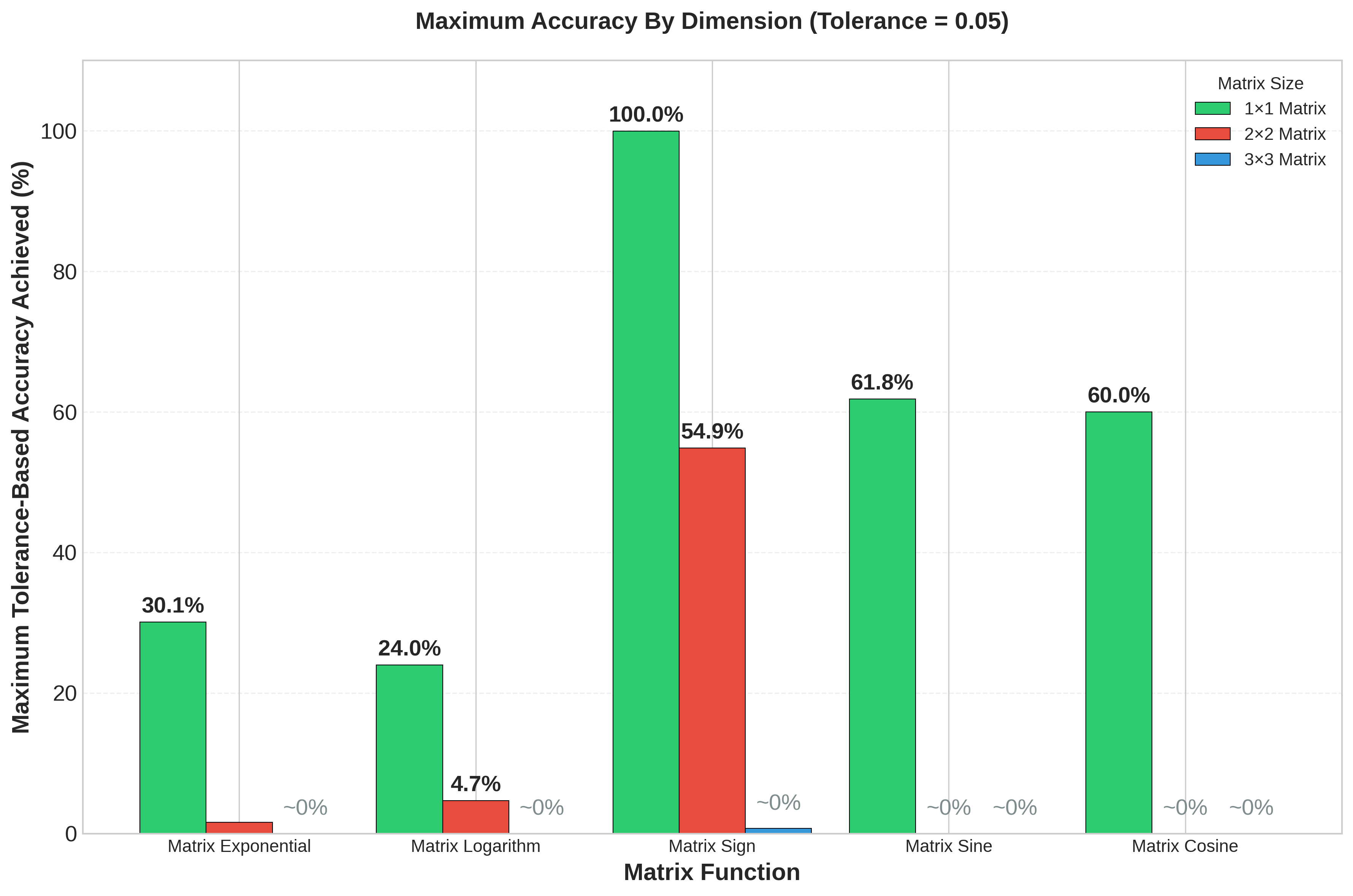}
        \caption{7-layer deep network}
        \label{fig:dnn_max_accuracy_by_dimension}
    \end{subfigure}
    \caption{Maximum tolerance-based accuracy decreases as matrix dimension increases. From dimension 3 onwards, accuracy is less than 3\% for all functions. Experiments were run up to dimension 8.}
    \label{fig:baseline_accuracy}
\end{figure}

\FloatBarrier

\subsection{Transformer Encoder-Decoder Results}
Our main numerical result is that the transformer encoder-decoder with the FP15 encoding approximates
the matrix sign function at 88.41\% accuracy at a tolerance of 0.01 for $3\times 3$ matrices.
Tables~\ref{tab:dim3_results} and~\ref{tab:dim5_results} report 
accuracy at several tolerance levels for $3\times 3$ and $5\times 5$ matrices respectively.
\begin{table}[t!]
    \centering
    \footnotesize
    \setlength{\tabcolsep}{3pt}
    \begin{tabular}{|l|c|c|c|c|c|}
    \hline
    \textbf{Operation} & \textbf{Encoding} & \multicolumn{4}{c|}{\textbf{Tolerance-Based Accuracy}} \\
    \cline{3-6}
    & & \textbf{tol=0.05} & \textbf{tol=0.02} & \textbf{tol=0.01} & \textbf{tol=0.005} \\
    \hline
    Exponential & P10 & 85.6\% & 24.18\% & 2.48\% & 0.13\% \\
    Logarithm & P10 & 93.09\% & 86.56\% & 74.96\% & 46.92\% \\
    Sign & P10 & 95.84\% & 87.58\% & 61.36\% & 27.26\% \\
    Sine & P10 & 2.92\% & 1.39\% & 0.23\% & 0.02\% \\
    Cosine & P10 & 5.16\% & 1.17\% & 0.11\% & 0.00\% \\
    \hline
    Exponential & P1000 & 88.92\% & 37.84\% & 9.9\% & 1.21\% \\
    Logarithm & P1000 & 92.52\% & 85.13\% & 72.09\% & 46.29\% \\
    Sign & P1000 & 96.26\% & 91.28\% & 76.53\% & 42.61\% \\
    Sine & P1000 & 5.13\% & 2.3\% & 0.19\% & 0.02\% \\
    Cosine & P1000 & 10.43\% & 5.28\% & 0.64\% & 0.00\% \\
    \hline
    Exponential & FP15 & 0.00\% & 0.00\% & 0.00\% & 0.00\% \\
    Logarithm & FP15 & 0.00\% & 0.00\% & 0.00\% & 0.00\% \\
    Sign & FP15 & \textbf{97.0\%} & \textbf{94.5\%} & \textbf{88.41\%} & \textbf{67.52\%} \\
    Sine & FP15 & 0.00\% & 0.00\% & 0.00\% & 0.00\% \\
    Cosine & FP15 & 0.00\% & 0.00\% & 0.00\% & 0.00\% \\
    \hline
    Exponential & B1999 & 99.58\% & 74.54\% & 24.79\% & 2.46\% \\
    Logarithm & B1999 & 93.02\% & 79.41\% & 55.68\% & 19.34\% \\
    Sign & B1999 & 44.85\% & 28.0\% & 21.26\% & 20.69\% \\
    Sine & B1999 & 0.00\% & 0.00\% & 0.00\% & 0.00\% \\
    Cosine & B1999 & 0.04\% & 0.00\% & 0.00\% & 0.00\% \\
    \hline
    \end{tabular}
    \caption{Accuracy results for different matrix functions and encodings on $3\times 3$ matrices across various error tolerances.}
    \label{tab:dim3_results}
\end{table}
\begin{table}[t!]
    \centering
    \footnotesize
    \setlength{\tabcolsep}{3pt}
    \begin{tabular}{|l|c|c|c|c|c|}
    \hline
    \textbf{Operation} & \textbf{Encoding} & \multicolumn{4}{c|}{\textbf{Tolerance-Based Accuracy}} \\
    \cline{3-6}
    & & \textbf{tol=0.05} & \textbf{tol=0.02} & \textbf{tol=0.01} & \textbf{tol=0.005} \\
    \hline
    Exponential & P10 & 0.00\% & 0.00\% & 0.00\% & 0.00\% \\
    Logarithm & P10 & 0.00\% & 0.00\% & 0.00\% & 0.00\% \\
    Sign & P10 & 0.00\% & 0.00\% & 0.00\% & 0.00\% \\
    Sine & P10 & 0.00\% & 0.00\% & 0.00\% & 0.00\% \\
    Cosine & P10 & 0.00\% & 0.00\% & 0.00\% & 0.00\% \\
    \hline
    Exponential & P1000 & 68.04\% & 11.24\% & 0.64\% & 0.00\% \\
    Logarithm & P1000 & 69.08\% & 31.89\% & 2.62\% & 0.00\% \\
    Sign & P1000 & 1.89\% & 1.89\% & 1.89\% & 1.89\% \\
    Sine & P1000 & 0.00\% & 0.00\% & 0.00\% & 0.00\% \\
    Cosine & P1000 & 0.00\% & 0.00\% & 0.00\% & 0.00\% \\
    \hline
    Exponential & FP15 & 0.00\% & 0.00\% & 0.00\% & 0.00\% \\
    Logarithm & FP15 & 0.00\% & 0.00\% & 0.00\% & 0.00\% \\
    Sign & FP15 & 1.19\% & 1.19\% & 1.19\% & 1.19\% \\
    Sine & FP15 & 0.00\% & 0.00\% & 0.00\% & 0.00\% \\
    Cosine & FP15 & 0.00\% & 0.00\% & 0.00\% & 0.00\% \\
    \hline
    Exponential & B1999 & 93.86\% & 32.17\% & 1.75\% & 0.00\% \\
    Logarithm & B1999 & 84.31\% & 57.03\% & 13.06\% & 0.03\% \\
    Sign & B1999 & 1.01\% & 1.01\% & 1.01\% & 1.01\% \\
    Sine & B1999 & 0.00\% & 0.00\% & 0.00\% & 0.00\% \\
    Cosine & B1999 & 0.00\% & 0.00\% & 0.00\% & 0.00\% \\
    \hline
    \end{tabular}
    \caption{Accuracy results for different matrix functions and encodings on $5\times 5$ matrices across various error tolerances.}
    \label{tab:dim5_results}
\end{table}
We find that the choice of encoding matters significantly. The FP15 encoding performs well for the matrix sign function
and the P10 encoding performs well at an accuracy of 74.96\% for the matrix logarithm function 
at a tolerance of 0.01 for $3\times 3$ matrices. The FP15 has the highest vocabulary size (30,000 words)
and the P10 has the lowest vocabulary size (210 words) in our experiments.
Every encoding fails to approximate the matrix sine and cosine functions.
As we have used $[-5, 5]$ as the range for the matrix entries, the matrix exponential has an extreme dynamic range.
Our baseline methods of the 3 and 7 layer neural networks were unable to approximate the matrix exponential
function for $3\times 3$ matrices at a tolerance of 0.05, higher than the tolerance of
0.01 which we use as a reasonable benchmark. However, we were able to approximate the matrix exponential
function for $3\times 3$ and $5\times 5$ matrices at a tolerance of 0.05 with the transformer encoder-decoder
with the B1999 encoding at a tolerance-based accuracy of 99.58\% and 93.86\% respectively which is a significant improvement over the baseline methods.

\section{Conclusions}
\label{sec:conclusions}

We presented two results on neural networks for matrix function 
approximation.
Theorem~\ref{thm:dnn_matrix_exp} shows ReLU DNNs can approximate 
the matrix exponential with width exponential in $nM$ and depth 
linear up to log factors. Despite the positive result in Theorem~\ref{thm:dnn_matrix_exp},
numerical experiments with ReLU DNNs were not satisfactory.
This gap is not surprising given that the architecture from the theorem was not implemented in the numerical experiments.
Moreover, the theorem is an existence result and does not involve training.
The main numerical result from Table~\ref{tab:dim3_results} shows that transformer encoder-decoders with 
numerical encodings can approximate the matrix sign function well for $3\times 3$ matrices at a tolerance of 0.01.
This is an improvement over the baseline methods, which were unable to approximate the matrix
functions used in the experiments for $3\times 3$ matrices.
Our results on approximating the matrix sign function connect to 
recent work on transformers for Lyapunov functions by Alfarano et al.~\cite{alfarano2024global}.
While that work targets symbolic solutions for nonlinear systems, our approach addresses numerical computation for linear systems.
Both directions point to the potential of transformers in stability analysis and control theory.
Finally, the strong dependence on encoding scheme suggests that numerical 
representations designed for specific problems could help.

\subsubsection{Future Work.}
Several questions remain open.
First, extending these results to larger matrices (dimension 10 or higher) is important for
practical applications but requires new ideas, given the sharp accuracy drop from dimension 3 to 5.
Second, the complete failure on matrix sine and cosine is not yet understood;
the oscillatory nature of these functions may pose a distinct challenge.
Third, while we provide approximation bounds for ReLU networks, theoretical analysis of transformer architectures for matrix functions is lacking.
Fourth, the strong dependence on encoding scheme is purely empirical.
Developing principled guidelines for selecting or designing encodings based on the target
function would be valuable.

\begin{credits}
\subsubsection{\ackname} 
This work utilized the computational resources provided by the Digital Research Alliance of Canada (DRAC).
SB acknowledges the support of NSERC through grant RGPIN-2020-06766.

\subsubsection{\discintname}
The authors have no competing interests to declare that are relevant to the content of this article.
\end{credits}


\bibliographystyle{splncs04}
\bibliography{mybibliography}

\end{document}